\documentclass[11pt,a4paper]{article}

\usepackage[margin=2.5cm]{geometry}
\usepackage{graphicx}
\usepackage{booktabs}
\usepackage{array}
\usepackage{amsmath,amssymb}
\usepackage{amsthm}
\usepackage[hidelinks]{hyperref}
\usepackage{xcolor}
\usepackage{tikz}
\usetikzlibrary{positioning,fit,arrows.meta,calc}

\newcolumntype{C}[1]{>{\centering\arraybackslash}p{#1}}

\title{\textbf{An Explicit World Model Based on Data-First Ontology:\\
DaoQL Multimodal Storage Validation and Counterfactual Reasoning Evaluation}}
\author{Zhanbo Li, Shifeng Wu, Xiangjin Meng, Wenjie Cai\\
Graduate Students, School of Mathematics and Statistics, Chongqing University\\
DaoQL Project Group\\
Contact affiliation: Guangzhou Polytechnic Normal University\\
\texttt{shifengwu@gpnu.edu.cn}\\
\texttt{zhanbo.lee@hotmail.com}\\
\texttt{202506131076T@stu.cqu.edu.cn}\\
\texttt{202506131075T@stu.cqu.edu.cn}}
\date{May 2026}

\begin{document}
\maketitle

\begin{abstract}
Large language models encode world models implicitly in neural weights, which exposes four structural risks in high-precision domains such as medicine and finance: hallucination, frozen knowledge, poor explainability, and poor modifiability. This paper proposes \emph{data-first ontology}: LLMs are treated as reasoning and language engines, while deterministic knowledge is moved into an explicit multimodal database, DaoQL. We formalize an explicit world model and show that, under rule independence, deterministic evaluation, and fixed conflict resolution, explicit models provide a sufficient condition for composable counterfactual decomposability; implicit models lack atomic read/$\delta$ semantics and therefore provide no comparable architectural guarantee. The implemented system focuses on DaoQL's verified storage layer and explicit Eval path, integrating graph, column, vector, and full-text engines within one process. KVCache graph nodes, expert hot updates, and the DaoQL-Agent runtime remain future work. On an embedded same-machine setup, DaoQL reports graph BFS at 1.20\,ms, HNSW at 83.1\,$\mu$s, and a Fluent hybrid query at 105.8\,$\mu$s; these results indicate engineering potential but must be interpreted with deployment-shape differences from client-server systems. Exploratory measurements on LDBC SNB SF1~\cite{ldbc2015snb} and ANN-Benchmarks~\cite{aumuller2020annbenchmarks} further show 34/34 query coverage with interactive-class queries mostly in the sub-millisecond to millisecond range, but only 1.8\,QPS overall due to long-tail BI/IC queries; ANN-Benchmarks reaches Recall@10 $\geq$ 99\% at thousand-level QPS after a bridge-edge protection fix. In a five-domain counterfactual experiment ($n=1250$), DaoQL+GPT-4o achieves 94\% composable counterfactual decomposability, 49 percentage points above GPT-4o alone. The paper explicitly separates provable structure, preliminary empirical evidence, and architectural roadmap claims.
\end{abstract}

\noindent\textbf{Keywords:} data-first ontology; multimodal database; neuro-symbolic AI; world model; counterfactual reasoning; benchmark evaluation

\section{Introduction}

Large language models (LLMs)~\cite{kaplan2020scaling,hoffmann2022chinchilla} show strong generalization and natural-language reasoning, but their default paradigm encodes a ``world model'' implicitly in billions of neural parameters~\cite{ha2018worldmodels,lecun2022ami}. This tight coupling of knowledge and computation creates four structural limitations in settings that require precision and auditability: factual conflicts caused by probabilistic sampling, frozen knowledge with high update cost, weak explainability of reasoning paths~\cite{bahdanau2015attention}, and poor surgical modifiability due to catastrophic interference~\cite{mccloskey1989catastrophic}. We argue at the operational-semantics level that these limitations share one gap: implicit models lack atomic read-by-ID and local update semantics.

RAG and agent frameworks such as LangChain, Hermes, and OpenClaw mitigate memory and knowledge defects by attaching vector stores or graphs to LLMs. This establishes a separation between the logic layer and the storage layer, but it still has three bottlenecks. First, token windows force lossy context compression. Second, memory retrieval and agent execution often live in separate processes, introducing synchronization and serialization overhead. Third, reasoning logs and database writes are usually recorded separately, making financial-grade auditability difficult. Existing systems do not provide read, $\delta$, $\mathsf{Eval}$, and $\mathsf{Trace}$ as unified atomic semantics.

We propose an explicit world-model architecture guided by \emph{data-first ontology}~\cite{garcez2023neurosymbolic,rocktaschel2017npor,serafini2016ltn,alphaproof2024deepmind}: the LLM is a reasoning and natural-language engine, while deterministic knowledge is stored in DaoQL, a multimodal database. DaoQL uses six explicit primitives - Being, Def, Type, Relation, Contract, and Version - to construct an addressable and versioned world model $W$. Theoretically, we formalize the operational-semantics difference between explicit and implicit models, derive a sufficient condition for composable counterfactual decomposability (Theorem~\ref{thm:cf-decomp}), and analyze why single-process multimodal execution can reduce latency. Empirically, the paper focuses on the verified storage layer, explicit Eval path, and counterfactual experiment; KVCache graph nodes and the full DaoQL-Agent runtime remain design-stage components.

The contributions are:
\begin{enumerate}
  \item \textbf{Theory.} We define an explicit world model, prove finite Def bootstrapping, strong Contract termination, and local update, and explain why implicit models lack an architectural guarantee for composable counterfactuals.
  \item \textbf{System.} We implement and evaluate DaoQL's multimodal storage layer, including mmap reads, relation-type CSR adjacency, WAL/MVCC, and cross-engine hybrid queries.
  \item \textbf{Evaluation.} We combine three-tier microbenchmarks, same-machine competitor comparisons, LDBC/ANN exploratory standard-set measurements, and a 1,250-query counterfactual experiment.
\end{enumerate}

\begin{table}[htbp]
  \centering
  \caption{Scope of claims in this paper}
  \label{tab:claim_scope}
  \small
  \begin{tabular}{p{3.3cm}p{5.1cm}p{3.0cm}}
    \toprule
    \textbf{Layer / capability} & \textbf{Evidence in this paper} & \textbf{Status} \\
    \midrule
    Explicit world-model theory & Formal definitions, well-formedness, counterfactual sufficient conditions & Established theoretically \\
    DaoQL storage layer & Graph, column, vector, full-text, WAL/MVCC, mmap benchmarks & Implemented and measured \\
    LDBC SNB / ANN sets & SF1 34/34 query coverage; ANN Recall@10 curve & Exploratory measurement \\
    Explicit Eval path & 1,250-query counterfactual experiment & Preliminary validation \\
    Memory-layer recall & HNSW + Tantivy + graph benchmarks as indirect support & Partially supported \\
    DaoQL-Agent runtime & Design document and context mechanism & Design stage \\
    KVCache graph nodes and expert hot updates & Architecture roadmap & Future work \\
    \bottomrule
  \end{tabular}
\end{table}

\section{Related Work}

\subsection{Knowledge Representation and Graph Databases}

Knowledge representation and reasoning (KRR)~\cite{mccarthy1959commonsense,newell1963gps,levesque1985krr,minsky1975frames,quillian1968semantic,pearl1988probabilistic,guarino1995ontology} aims to make entities and logic machine-operable. RDF~\cite{w3c2014rdf} and OWL~\cite{w3c2012owl,w3c2013sparql} provide the logical foundation of the Semantic Web, but mainly describe static propositions and do not natively integrate version chains or executable behavioral semantics. Datomic~\cite{datomic2012} pioneered immutable append and ``time as a first-class concern,'' but its EAVT model remains a passive fact-recording model rather than an autonomous computational unit model.

Property graphs such as Neo4j~\cite{neo4j2024,webber2013graphdb} and strongly typed graph databases such as TypeDB~\cite{typedb2024} improve structured graph modeling. However, for next-generation AI world models they still face two limitations: inheritance and schema are largely static, and vector retrieval and column aggregation are not unified in the same process, causing network and serialization overhead. DaoQL addresses this design space by embedding typing, graph traversal, vector search, column aggregation, versioning, and contract checks into one storage kernel.

\subsection{LLM-Augmentation and Agent Frameworks}

RAG systems increasingly attach external knowledge bases to LLMs. GraphRAG~\cite{edge2024graphrag} and LightRAG~\cite{guo2024lightrag} extract entities and relations from text and use graph topology to improve context. Yet the graph remains an external, mostly read-only augmentation component. When injected facts conflict with parameterized knowledge, systems still lack deterministic, atomic conflict resolution and counterfactual evaluation.

Long-context memory systems such as MemGPT~\cite{packer2023memgpt} manage conversation history using a virtual-memory analogy, but operate at the text-block granularity rather than as structured entities with version chains. Agent frameworks typically use LLMs as controllers and external databases such as SQLite, Chroma, or Pinecone as memory components. This is easy to engineer but fragments context, memory, and audit logs across systems. DaoQL-Agent proposes representing turns and messages as Being nodes and using graph traversal instead of lossy compaction; however, the full agent runtime is not measured in this paper.

\subsection{Industrial Ontology Platforms and Standard Benchmarks}

Industrial Ontology platforms, exemplified by Palantir Foundry~\cite{palantir2024ontology}, use Object/Link/Action semantics to model enterprise entities and operations. They share the paper's motivation toward explicit semantics, but typically place Ontology as a semantic mapping layer above multiple stores such as Cassandra, Elasticsearch, and JanusGraph. Their advantage is federated access to existing data infrastructure. DaoQL instead sinks Being/Type/Relation into single-process storage primitives, trading migration cost for a tighter Eval/Trace path. LDBC SNB~\cite{ldbc2015snb} and ANN-Benchmarks~\cite{aumuller2020annbenchmarks} provide common workloads for graph and approximate nearest-neighbor systems; we use them as exploratory standard-set measurements rather than official certification.

\section{Materials and Methods}

This section is organized as theory, implementation, and experimental design. Section~3.1 formalizes the explicit world model; Section~3.2 describes the system architecture; Section~3.3 specifies benchmarks and counterfactual evaluation.

\subsection{Theory and Formal Model}

\subsubsection{Overview of Data-First Ontology}

Data-first ontology claims that a world model should consist of structured, inspectable, controllable explicit data rather than implicit weights. DaoQL realizes this through six primitives. \textbf{Being} is the basic entity unit and is identified by UUID. \textbf{Def} and \textbf{Type} define fields, constraints, relationships, and runtime dispatch. \textbf{Relation} connects Beings as directed edges. \textbf{Contract} is a fuel-limited Rhai script embedded in the write path. \textbf{Version} is a hash-linked history record that supports MVCC and audit replay. The following formal framework gives the precise definitions and proof obligations.


\newtheorem{definition}{Definition}
\newtheorem{proposition}{Proposition}
\newtheorem{theorem}{Theorem}
\newtheorem{observation}{Observation}

To turn the Introduction's ``knowledge sinking'' claim into testable scientific statements, this section formalizes the explicit world model and derives its structural properties; system implementation (\S3.2) and experimental design (\S3.3) build on this foundation.

\subsubsection{Formal Definitions}

An implicit world model encodes knowledge in a neural network $f_\theta:\mathcal{X}\to\mathcal{Y}$~\cite{ha2018worldmodels,lecun2022ami}, with forward inference and global parameter updates $\theta\mapsto\theta'$ as primary operations. An explicit world model, by contrast, carries deterministic knowledge in addressable, versioned data structures. Table~\ref{tab:symbols} summarizes the main symbols.

\begin{table}[htbp]
  \centering
  \caption{Main symbols in the theory framework}
  \label{tab:symbols}
  \small
  \begin{tabular}{cl}
    \toprule
    \textbf{Symbol} & \textbf{Meaning} \\
    \midrule
    $\mathcal{ID},\mathcal{B},\mathcal{D},\mathcal{T},\mathcal{R},\mathcal{C},\mathcal{V}$ & Identifier and six primitive sets \\
    $W=(\mathcal{B},\mathcal{D},\mathcal{T},\mathcal{R},\mathcal{C},\mathcal{V})$ & Explicit world-model state \\
    $\delta(W,p)$ & State after applying patch $p$ to $W$ \\
    $\mathsf{Eval}(\phi,W)$ & Evaluation of query/rule $\phi$ on state $W$ \\
    $\mathsf{Trace}(q,W)$ & Execution trace of query $q$ on $W$ \\
    $F_{\max}$ & Contract fuel limit \\
    \bottomrule
  \end{tabular}
\end{table}

\begin{definition}[Six Primitives]
\label{def:six-primitives}
The primitives are:
\begin{itemize}
  \item \textbf{Being} $b=(id,core,ext,history)$: $id\in\mathcal{ID}$, $history=[v_1,\ldots,v_n]$ is a version chain;
  \item \textbf{Def} $d=(name,fields,constraints,contracts,fsm,rel\_decls)$;
  \item \textbf{Type} $t=(name,parent,contracts,constraints)$;
  \item \textbf{Relation} $r=(type,src,dst,direction,ext,window)$;
  \item \textbf{Contract} $c=(hook,source,ast)$;
  \item \textbf{Version} $v=(ts,hash,prev\_hash,payload)$.
\end{itemize}
\noindent The explicit world model is $W=(\mathcal{B},\mathcal{D},\mathcal{T},\mathcal{R},\mathcal{C},\mathcal{V})$.
\end{definition}

\begin{definition}[State Equivalence]
Given a fact domain $\mathcal{F}$, if $W_1$ and $W_2$ agree on all observable facts in $\mathcal{F}$, we write $W_1\equiv_{\mathcal{F}}W_2$.
\end{definition}

This structure realizes four atomic semantics: read by ID, patch update ($\delta$), deterministic rule evaluation ($\mathsf{Eval}$), and execution tracing ($\mathsf{Trace}$). DaoQL implements them via BeingId addressing, append-only WAL, and MVCC snapshots (\S3.2.3).

\subsubsection{Well-Formedness and Execution Safety}

DaoQL adopts a ``schema-as-data'' design: Def is itself a Being, bootstrapped by the axiom Def \texttt{DAO\_DEF}. We must ensure that this mechanism does not yield infinite meta-definition chains and that Contracts embedded in the write path cannot violate system termination.

\begin{proposition}[Finite Def Bootstrap Chain]
\label{prop:bootstrap}
Let axiom Def $d_0\in\mathcal{D}$ satisfy $d_0.core.def=d_0$; and $\forall d\in\mathcal{D}\setminus\{d_0\},\, d.core.def=d_0$, $\forall b\in\mathcal{B},\, b.core.def\in\mathcal{D}\setminus\{d_0\}$. Then any ordinary Being has def-reference chain length at most 2, any ordinary Def has chain length 1, and \texttt{DAO\_DEF} is only a controlled self-reference.
\end{proposition}

\begin{proof}
For any $b\in\mathcal{B}$, $b.core.def=d_1\in\mathcal{D}\setminus\{d_0\}$ and $d_1.core.def=d_0$, hence $b\to d_1\to d_0$ with chain length 2. For any $d\in\mathcal{D}\setminus\{d_0\}$, $d\to d_0$ with chain length 1. Suppose an infinite chain $d_0\to d_1\to d_2\to\cdots$ exists: if $d_1\ne d_0$, then $d_1$ is an ordinary Def and $d_1.core.def=d_0$ forces termination at $d_0$, a contradiction.
\end{proof}

\begin{proposition}[Strong Contract Termination]
\label{prop:contract}
Suppose Contract execution forbids new writes and external I/O, consumes one fuel unit per step, and is capped at $F_{\max}$. Then $\forall c\in\mathcal{C}$, $\mathsf{Exec}(c)$ halts within $F_{\max}$ steps.
\end{proposition}

\begin{proof}
Use remaining fuel $\Phi(e_t)=\mathsf{Fuel}(e_t)$ as a potential function, initialized to $F_{\max}$ and strictly decreasing by one per step. Forbidding write recursion and external I/O excludes infinite-wait paths; by well-ordering of natural numbers, execution terminates and rolls back when fuel reaches zero, so at most $F_{\max}$ steps occur.
\end{proof}

Propositions~\ref{prop:bootstrap}--\ref{prop:contract} bound the metadata layer and guarantee write-path termination, supporting the \texttt{DAO\_DEF} axiom layer and the default 10{,}000-step Contract fuel limit described in \S3.1.1.

\subsubsection{Operational Semantics: Implicit vs.\ Explicit Models}

The four LLM crises named in the Introduction can be unified at the operational-semantics level: implicit model $\mathcal{O}_{\mathrm{imp}}=\{f_\theta,\,\theta\mapsto\theta'\}$ lacks database-style atomic read and $\delta$.

\begin{observation}[Implicit Models Lack Database-Style Atomic Operations]
\label{obs:implicit}
Operationally, LLM knowledge is encoded in distributed parameters $\theta$, with usable operations mainly $f_\theta(x)$ and $\theta\mapsto\theta'$; therefore existing LLM architectures do not provide database-isomorphic $\mathsf{read}_\theta(id)$ and $\delta_\theta(\theta,patch_{id})$, nor can precise read-by-ID and local update be system invariants.
\end{observation}

\begin{proof}
If a $\delta_\theta$ existed that updated only the fact for $id$, parameters should decompose as $\theta=\bigsqcup_{id}\theta|_{id}$; but LLM weights encode semantics in a distributed manner, with no provably independent $\theta|_{id}$. Fine-tuning $\theta'=\theta-\eta\nabla_\theta L$ is global and aligns with catastrophic interference~\cite{mccloskey1989catastrophic}. Prompt intervention $y=f_\theta(x')$ changes input only, not persistent state. The assumption fails.
\end{proof}

\begin{proposition}[Local Update in Explicit Models]
\label{prop:local}
Let $W'=\delta(W,p)$, where the scope of $p$ is $\mathsf{Dom}(p)\subseteq\{id_0\}\cup\{r\mid r.src=id_0\ \mathrm{or}\ r.dst=id_0\}$, and the system uses BeingId addressing, append-only WAL, and MVCC snapshot reads. Then $\forall b\in\mathcal{B},\, b.id\ne id_0:\, W'|_b\equiv W|_b$.
\end{proposition}

\begin{proof}
The patch does not touch other entity fields; WAL appends only version $v_{n+1}$ for $id_0$, leaving $history(b)$, $core(b)$, and $ext(b)$ unchanged for all other entities. By the definition of state equivalence, semantics on the untouched fact domain is unchanged.
\end{proof}

\begin{observation}[Shared Operational Root of the Four LLM Crises]
\label{obs:four-crises}
If a system has only $\mathcal{O}_{\mathrm{imp}}$ without $\mathsf{read}$, $\delta$, $\mathsf{Eval}$, and $\mathsf{Trace}$, then hallucination resistance (needs deterministic $\mathsf{Eval}$), live updates (needs $\delta$), explainability (needs $\mathsf{Trace}$), and surgical edits (needs local $\delta$) cannot simultaneously hold as architectural invariants.
\end{observation}

\begin{proof}
By Observation~\ref{obs:implicit}, implicit models lack local $\delta$ and atomic read. $f_\theta$ outputs token distributions rather than $\mathsf{Eval}(\phi,S)\in\{\top,\bot\}$~\cite{ji2023hallucination}; $\theta\mapsto\theta'$ is costly and global; internal activations cannot be reconstructed as symbolic reasoning chains. The operations required by all four properties are absent from $\mathcal{O}_{\mathrm{imp}}$, hence they share one operational root.
\end{proof}

Observation~\ref{obs:four-crises} does not claim that every LLM failure has a single explanation; rather, in settings requiring deterministic factual judgment, local update, and audit trails, the four crises share one operational gap. Data-first ontology fills these atomic semantics through explicit structure $W$, forming the theoretical starting point of our architecture.

\subsubsection{Counterfactual Reasoning}
\label{subsec:cf-theory}

In Pearl's causal hierarchy, the counterfactual layer asks how conclusions would change had facts differed~\cite{pearl1988probabilistic}. We operationalize this as applying patch $p$ to explicit state $S=(W,\mathcal{K})$ and executing $\mathsf{Eval}(\phi,\delta(S,p))$, which directly matches the counterfactual protocol in \S3.3.

\begin{theorem}[Composable Counterfactual Decomposability]
\label{thm:cf-decomp}
Let $p_1,p_2$ be patches composed of finitely many atomic update/relate operations, and let $\phi$ depend on fact domain $\mathcal{F}_\phi$. If $\delta(\delta(S,p_1),p_2)\equiv_{\mathcal{F}_\phi}\delta(\delta(S,p_2),p_1)$, $\mathsf{Eval}$ is deterministic, and rule conflicts have fixed priority, then
\begin{equation}
\label{eq:cf-decomp}
\mathsf{Eval}(\phi,\delta(\delta(S,p_1),p_2))
=
\mathsf{Eval}(\phi,\delta(S,p_1\oplus p_2)).
\end{equation}
\end{theorem}

\begin{proof}
Let $S_2=\delta(\delta(S,p_1),p_2)$ and $S_{12}=\delta(S,p_1\oplus p_2)$. By patch commutativity and fixed conflict resolution, $S_2\equiv_{\mathcal{F}_\phi}S_{12}$. Since $\phi$ depends only on $\mathcal{F}_\phi$, evaluation is a deterministic function $g_\phi(S|_{\mathcal{F}_\phi})$, hence $\mathsf{Eval}(\phi,S_2)=\mathsf{Eval}(\phi,S_{12})$.
\end{proof}

\begin{observation}[No Architectural Guarantee of Composable Counterfactuals for Implicit Models]
\label{obs:cf-imp}
LLM implicit evaluation $\hat\phi(\ell,S)=Sample(f_\theta(\ell\oplus encode(S)))$ cannot architecturally guarantee that semantically equivalent phrasings yield the same conclusion, that composed edits satisfy Theorem~\ref{thm:cf-decomp}, or that edit order is immaterial.
\end{observation}

\begin{proof}
On the explicit path, semantically equivalent phrasings $\ell_1,\ell_2$ map to the same patch $p=\pi(\ell_1)=\pi(\ell_2)$; on the implicit path, $\ell_1\ne\ell_2$ may yield $P_\theta(y|\ell_1,S)\ne P_\theta(y|\ell_2,S)$. Implicit models lack $\delta$, so comparable structured states $S_2$ and $S_{12}$ cannot be constructed; prompt order changes attention context, so commutativity has no architectural guarantee.
\end{proof}

Theorem~\ref{thm:cf-decomp} shows that under rule independence and deterministic Eval, composable counterfactual decomposability is a provable structural property; Observation~\ref{obs:cf-imp} explains why the pure LLM path lacks an equivalent guarantee. Together they set the theoretical expectations for \S4.1.

\subsubsection{Neuro-Symbolic Division of Labor and Performance Mechanism}

In a neuro-symbolic architecture, deterministic knowledge $G\subseteq\mathcal{S}$ lives on the symbolic side; queries are evaluated by $\mathsf{Eval}(\phi,\mathcal{S})$ and the LLM only naturalizes $y=\pi_{nl}(r)$. The executor records $\mathsf{Trace}=\langle op_1,\ldots,op_m\rangle$ with Being/Relation IDs accessed at each step, so conclusions trace back to explicit reasoning paths---explaining 100\% traceability in the counterfactual experiment. This property fails when the LLM bypasses Eval and generates directly.

For performance, hybrid query $q=q_{\mathrm{vec}}\circ q_{\mathrm{graph}}\circ q_{\mathrm{col}}$ in a multi-stack architecture requires $k$ cross-process calls. With per-call network/IPC latency $T_{\mathrm{net},i}\ge\tau_{\mathrm{net}}>0$,
\begin{equation}
\label{eq:latency}
T_{\mathrm{multi}}-T_{\mathrm{mono}}\ge k\tau_{\mathrm{net}}-T_{\mathrm{dispatch}},
\end{equation}
where $T_{\mathrm{multi}}=\sum_i(T_{\mathrm{net},i}+T_{\mathrm{ser},i}+T_{\mathrm{query},i})$ and $T_{\mathrm{mono}}=T_{\mathrm{vec}}+T_{\mathrm{graph}}+T_{\mathrm{col}}+T_{\mathrm{dispatch}}$. In DaoQL's single process $T_{\mathrm{net}}=0$, eliminating cross-system overhead on the order of $k\tau_{\mathrm{net}}$; this inequality explains low hybrid-query latency but does not imply DaoQL is asymptotically optimal on all algorithms.

For graph traversal, DaoQL pointer-adjacency BFS on subgraph $G'=(V',E')$ runs in $O(|V'|+|E'|)$ because each vertex and edge is processed at most once. This shows the storage shape approximates the ideal BFS access model; measured 4.2--6.7$\times$ speedup over Neo4j (Table~\ref{tab:competitor}) is an engineering constant factor and must not be conflated with the complexity result.

We distinguish two kinds of claims: Contract termination, local update, and Eq.~(\ref{eq:cf-decomp}) are provable under stated axioms; BFS 1.20\,ms and 94\% composable decomposability are empirical results on specific hardware and datasets, validating predictions rather than replacing theory.

\subsection{System Architecture and Implementation}

\subsubsection{Five-Layer Integrated Architecture}

Figure~\ref{fig:five_layer} illustrates how an explicit world model could grow into a full neuro-symbolic system. The empirically verified core in this paper is the storage layer and the explicit Eval path; several memory, learning, inference, and application components remain roadmap items.

\begin{itemize}
  \item \textbf{Storage layer.} DaoQL integrates graph, column, vector, and full-text engines in one process with a unified BeingId. WAL and MVCC provide atomicity and crash consistency.
  \item \textbf{Memory layer.} Session lifecycle and historical recall are intended to use HNSW, Tantivy, and graph traversal. These components are individually measured, but the full agent memory layer is not.
  \item \textbf{Learning layer.} Expert and router hot updates are planned architecture goals.
  \item \textbf{Inference layer.} A customized LLM can in principle navigate semantically related KVCache graph nodes, but this remains a planned direction.
  \item \textbf{Application layer.} DaoQL-Agent exposes FSM-driven workflows and tool-safety Contracts, but the runtime is design-stage.
\end{itemize}

\begin{figure}[htbp]
  \centering
  \begin{tikzpicture}[
    node distance=0.55cm and 0.6cm,
    box/.style={draw, rounded corners=2pt, minimum width=2.0cm, minimum height=0.65cm, align=center, font=\small},
    layer/.style={draw, dashed, inner sep=7pt, rounded corners=4pt},
    arr/.style={->, >=stealth, thick}
  ]
    \node[box, fill=gray!12] (app) {Application\\DaoQL-Agent\\design stage};
    \node[layer, fit=(app), label={[font=\bfseries]above:Application Layer}] {};

    \node[box, fill=gray!12, below=1.0cm of app] (inf) {Inference\\V4 + Attention\\planned};
    \node[layer, fit=(inf), label={[font=\bfseries]above:Inference Layer}] {};

    \node[box, fill=gray!12, below=1.0cm of inf, xshift=-2.2cm] (learn) {Learning\\Expert/Router\\planned};
    \node[box, fill=cyan!10, right=0.5cm of learn] (mem) {Memory\\hybrid recall\\partially verified};
    \node[layer, fit=(learn)(mem), label={[font=\bfseries]above:Memory / Learning}] {};

    \node[box, fill=orange!10, below=1.0cm of mem, minimum width=10cm] (stor) {Storage DaoQL: Graph + Column + Vector + Fulltext + WAL/MVCC (unified BeingId)};
    \node[layer, fit=(stor), label={[font=\bfseries]above:Storage Layer}] {};

    \draw[arr] (app.south) -- (inf.north);
    \draw[arr] (inf.south) -- ++(0,-0.35) -| (mem.north);
    \draw[arr] (inf.south) -- ++(0,-0.35) -| (learn.north);
    \draw[arr] (mem.south) -- (stor.north);
    \draw[arr] (learn.south) -- (stor.north);
    \node[below=0.25cm of stor, font=\small\itshape, text width=12cm, align=center] {Data flow: user query $\rightarrow$ hybrid recall $\rightarrow$ mmap zero-copy load $\rightarrow$ inference $\rightarrow$ atomic WAL commit};
  \end{tikzpicture}
  \caption{Five-layer neuro-symbolic architecture and verification scope. The storage layer is measured in the current configuration; the memory layer is indirectly supported by component benchmarks; application, inference, and learning layers are future targets.}
  \label{fig:five_layer}
\end{figure}
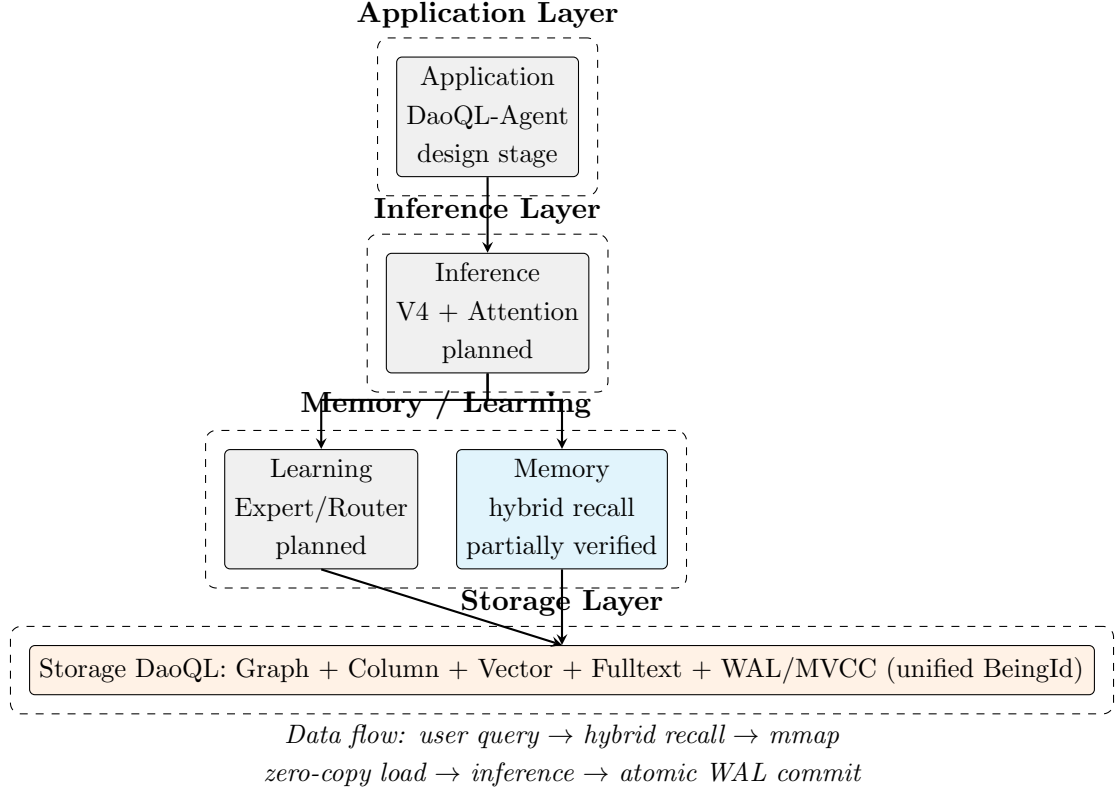

In the complete design, a user query triggers hybrid recall, locates relevant Beings and KVCache offsets, loads them through mmap, invokes inference, and writes results atomically through the WAL. The verified part is concentrated in storage: the current configuration bypasses PageCache/ArcSwap on graph reads, and WAL recovery for 10K entries takes 219.6\,$\mu$s.

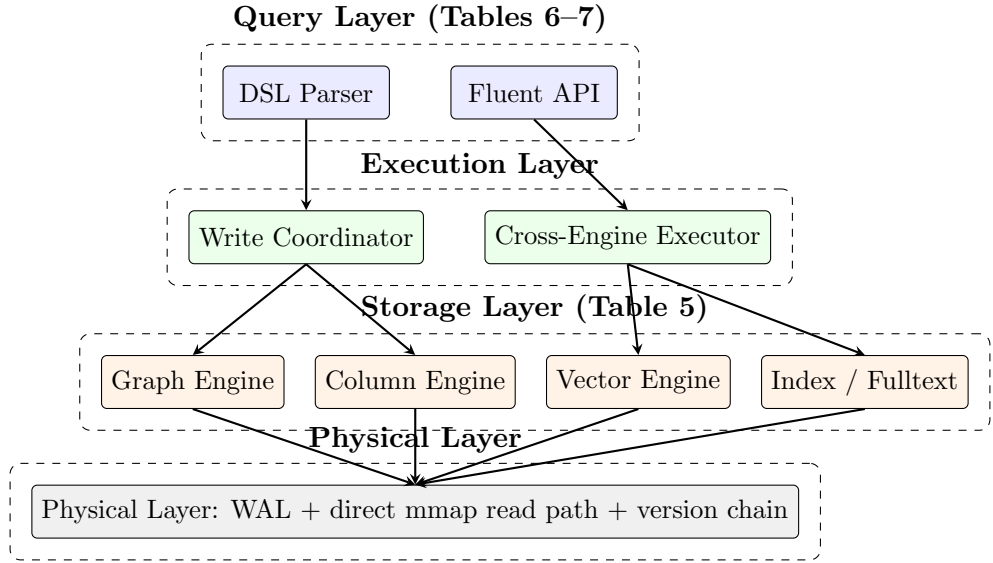
\begin{figure}[htbp]
  \centering
  \begin{tikzpicture}[
    node distance=0.6cm and 0.8cm,
    box/.style={draw, rounded corners=2pt, minimum width=2.2cm, minimum height=0.7cm, align=center, font=\small},
    layer/.style={draw, dashed, inner sep=8pt, rounded corners=4pt},
    arr/.style={->, >=stealth, thick}
  ]
    \node[box, fill=blue!8] (dsl) {DSL Parser};
    \node[box, fill=blue!8, right=of dsl] (fluent) {Fluent API};
    \node[layer, fit=(dsl)(fluent), label={[font=\bfseries]above:Query Layer (Tables~\ref{tab:fluent_micro}--\ref{tab:dsl_micro})}] (ql) {};

    \node[box, fill=green!8, below=1.2cm of dsl] (write) {Write Coordinator};
    \node[box, fill=green!8, right=of write] (query) {Cross-Engine Executor};
    \node[layer, fit=(write)(query), label={[font=\bfseries]above:Execution Layer}] (el) {};

    \node[box, fill=orange!10, below=1.2cm of write, xshift=-1.5cm] (graph) {Graph Engine};
    \node[box, fill=orange!10, right=0.4cm of graph] (column) {Column Engine};
    \node[box, fill=orange!10, right=0.4cm of column] (vector) {Vector Engine};
    \node[box, fill=orange!10, right=0.4cm of vector] (index) {Index / Fulltext};
    \node[layer, fit=(graph)(index), label={[font=\bfseries]above:Storage Layer (Table~\ref{tab:engine_micro})}] (sl) {};

    \node[box, fill=gray!12, below=1.0cm of column, minimum width=8cm] (wal) {Physical Layer: WAL + direct mmap read path + version chain};
    \node[layer, fit=(wal), label={[font=\bfseries]above:Physical Layer}] (pl) {};

    \draw[arr] (dsl.south) -- (write.north);
    \draw[arr] (fluent.south) -- (query.north);
    \draw[arr] (write.south) -- (graph.north);
    \draw[arr] (write.south) -- (column.north);
    \draw[arr] (query.south) -- (vector.north);
    \draw[arr] (query.south) -- (index.north);
    \draw[arr] (graph.south) -- (wal.north);
    \draw[arr] (column.south) -- (wal.north);
    \draw[arr] (vector.south) -- (wal.north);
    \draw[arr] (index.south) -- (wal.north);
  \end{tikzpicture}
  \caption{DaoQL's internal four-layer engine architecture and its mapping to the three-tier benchmark suite. The current measurement introduces direct mmap reads in the physical layer and P0\#1/\#2/\#5 write-path optimizations.}
  \label{fig:architecture}
\end{figure}

\subsubsection{Graph-Driven Agent Context and Direct mmap Reads}

External agent frameworks usually compress old messages because of token-window limits. DaoQL-Agent instead models interaction history as graph nodes connected by relations such as \texttt{HAS\_TURN} and \texttt{HAS\_MESSAGE}. A ContextBuilder can traverse from a session node to assemble context. The measured graph traversal latency (BFS 1.20\,ms and DSL Chain Depth30 79.3\,$\mu$s) supports the feasibility of this design, but not a complete agent-system claim.

For graph-heavy workloads, the current configuration changes NodeStore/EdgeStore reads from pagecache + ArcSwap to direct mmap pointer reads. It further maintains a per-RelationType CSR adjacency file (\texttt{csr\_\{:04x\}.bin}) so type-filtered BFS and triangle-count queries can read neighbors without scanning edge chains. This reduces edge write-read latency from 397.6\,ns to 304.4\,ns and BFS fanout4 depth5 from 1.97\,ms to 1.20\,ms. In LDBC SNB, IC14 drops from the 30\,s range to 0.64\,ms.

\subsection{Experimental Design}

\subsubsection{Environment and Datasets}

Microbenchmarks and same-machine competitor comparisons run on an Apple M-series machine with 64\,GB LPDDR5 and NVMe SSD ($\sim$3.5\,GB/s), Rust 1.94.1, Criterion.rs release mode, and median reporting. LDBC SNB and ANN-Benchmarks exploratory runs use Apple M4 Max, 128\,GB LPDDR5, and NVMe SSD ($\sim$7\,GB/s), DaoQL v1.1.0 release build. The two groups are reported separately and are not merged into a single ranking.

\begin{itemize}
  \item \textbf{XY-ERP}: synthetic ERP data, about 1.9M entities covering orders, BOM, and supplier links.
  \item \textbf{YAGO4}: public knowledge graph, about 3.4M entities, used for reproducibility checks.
  \item \textbf{Synthetic microbenchmarks}: 10K vectors (64D, Cosine, $M=16$, ef=100) and 10K Chinese documents.
  \item \textbf{LDBC SNB SF1}: about 1M Persons, 37M edges, and 34 IC/BI queries, using a custom driver rather than official submission.
  \item \textbf{ANN-Benchmarks}: \texttt{random-s-100-euclidean}, 90K $\times$ 100 training vectors, 10K $\times$ 100 test vectors.
\end{itemize}

\begin{table}[htbp]
  \centering
  \caption{Fairness boundaries of performance comparisons}
  \label{tab:benchmark_fairness}
  \small
  \begin{tabular}{p{2.8cm}p{4.3cm}p{5.0cm}}
    \toprule
    \textbf{Dimension} & \textbf{Setting} & \textbf{Interpretation boundary} \\
    \midrule
    Deployment shape & DaoQL embedded; competitors often client-server & Ratios include process, protocol, and serialization costs \\
    Data scale & 10K--100K microbenchmarks; 1.9M business entities & Not evidence for billion-scale workloads \\
    Cache state & Criterion release median; same-machine repeated runs & Hot-cache and tail latency need separation \\
    HNSW parameters & 10K vectors, 64D, $M=16$, ef=100 & Only this parameter range \\
    Hybrid query & Single-process vector $\rightarrow$ graph $\rightarrow$ column & No full multi-stack RAG baseline \\
    Standard sets & M4 Max 128GB vs M-series 64GB & Reported separately, not merged \\
    LDBC status & 34/34 queries run & Not official LDBC certification \\
    \bottomrule
  \end{tabular}
\end{table}

\subsubsection{Benchmark and Counterfactual Protocols}

We use three benchmark tiers: sub-crate engine microbenchmarks, Fluent API microbenchmarks, and DSL end-to-end microbenchmarks. Business stress tests include engine, Fluent, and DSL suites. Configuration ablation compares the previous pagecache/ArcSwap configuration with the current mmap + P0 write-path configuration. LDBC SNB uses SF1, 8 workers, 30\,s warmup, 600\,s measurement, and 1\,GB page cache. ANN-Benchmarks reports QPS and Recall@10 for ef\_search $\in\{50,100,200,400,800,1600\}$.

The counterfactual experiment compares GPT-4o with DaoQL+GPT-4o. It covers five domains - medical decision, financial credit, supply-chain risk, regulatory compliance, and dynamic pricing - with 50 cases per domain and five natural-language paraphrases per case, for 1,250 queries. Metrics include answer consistency, safety, reasoning-chain traceability, update correctness after new information, counterfactual stability, and composable counterfactual decomposability. The materials describe domains, scale, metrics, and partial ground-truth rules, but not full prompts, temperature, scoring protocol, or per-case error logs; therefore we treat the result as preliminary evidence.

\section{Results and Analysis}

We evaluate four aspects: counterfactual robustness, storage performance, same-machine and standard-set benchmarks, and system-level primitive composition. Values in Sections~4.2--4.3 use the M-series 64\,GB platform; Section~4.4 uses the M4 Max 128\,GB platform. All DaoQL measurements use real engine calls and no mock or stub paths.

\subsection{Counterfactual Consistency}

This experiment tests the prediction of Section~\ref{subsec:cf-theory}: if natural language maps to a patch and executes $\mathsf{Eval}(\phi,\delta(S,p))$, composed changes should satisfy Eq.~\ref{eq:cf-decomp}; GPT-4o alone lacks $\delta$ and structured state.

\begin{table}[htbp]
  \centering
  \caption{Counterfactual consistency experiment across five domains ($n=1250$)}
  \label{tab:counterfactual}
  \small
  \begin{tabular}{lrrr}
    \toprule
    \textbf{Metric} & \textbf{GPT-4o} & \textbf{DaoQL+GPT-4o} & \textbf{Gain} \\
    \midrule
    Answer consistency & 71\% & 97\% & +26\% \\
    Safety & 79\% & 95\% & +16\% \\
    Traceable reasoning chain & --- & 100\% & --- \\
    Correct update after new information & 61\% & 100\% & +39\% \\
    Counterfactual stability & 58\% & 96\% & +38\% \\
    Composable counterfactual decomposability & 45\% & 94\% & +49\% \\
    \bottomrule
  \end{tabular}
\end{table}

The gap between 94\% and 45\% on composable decomposability is consistent with Theorem~\ref{thm:cf-decomp} and Observation~\ref{obs:cf-imp}. The remaining 6\% failures in the explicit path are attributed in the source materials to parsing errors, under-encoded rules, or DSL boundaries; without per-case error logs, we do not further decompose the failure modes quantitatively.

\subsection{Three-Tier Microbenchmarks and Ablation}

Tables~\ref{tab:engine_micro}--\ref{tab:dsl_micro} compare the previous configuration with the current mmap read path and P0 write-path optimizations.

\begin{table}[htbp]
  \centering
  \caption{Engine-level ablation: previous vs.\ current configuration}
  \label{tab:engine_micro}
  \footnotesize
  \begin{tabular}{lrrr}
    \toprule
    \textbf{Benchmark} & \textbf{Previous} & \textbf{Current} & \textbf{Change} \\
    \midrule
    graph / node\_store\_alloc\_1000 & 10.81\,$\mu$s & 7.17\,$\mu$s & +34\% \\
    graph / node\_store\_write\_read & 1.86\,$\mu$s & 1.24\,$\mu$s & +33\% \\
    graph / edge\_store\_write\_read & 397.6\,ns & 304.4\,ns & +23\% \\
    column / raw\_append\_1000 & 9.03\,ms & 4.16\,ms & +54\% \\
    column / projected\_scan\_10000 & 377.2\,$\mu$s & 221.7\,$\mu$s & +41\% \\
    write / append\_single & 16.70\,$\mu$s & 4.74\,$\mu$s & +72\% \\
    write / append\_batch\_10000 & 50.0\,ms & 14.91\,ms & +70\% \\
    vector / hnsw\_search\_k10 & 95.4\,$\mu$s & 61.3\,$\mu$s & +36\% \\
    vector / persistence\_roundtrip & 11.13\,ms & 1.66\,ms & +85\% \\
    lookup (B+Tree) & 253.3\,ns & 160.7\,ns & +37\% \\
    \bottomrule
  \end{tabular}
\end{table}

\begin{table}[htbp]
  \centering
  \caption{Fluent API microbenchmark ablation}
  \label{tab:fluent_micro}
  \small
  \begin{tabular}{lrrr}
    \toprule
    \textbf{Benchmark} & \textbf{Previous} & \textbf{Current} & \textbf{Change} \\
    \midrule
    graph\_write\_and\_traverse & 18.26\,$\mu$s & 16.57\,$\mu$s & +9\% \\
    column\_write\_and\_scan & 389.4\,$\mu$s & 374.4\,$\mu$s & +4\% \\
    vector\_search & 110.6\,$\mu$s & 98.4\,$\mu$s & +11\% \\
    fulltext\_search & 482.0\,$\mu$s & 415.6\,$\mu$s & +14\% \\
    index\_uuid\_lookup & 178.1\,ns & 186.9\,ns & $-$5\% \\
    \bottomrule
  \end{tabular}
\end{table}

\begin{table}[htbp]
  \centering
  \caption{DSL microbenchmark ablation}
  \label{tab:dsl_micro}
  \small
  \begin{tabular}{lrrr}
    \toprule
    \textbf{Benchmark} & \textbf{Previous} & \textbf{Current} & \textbf{Change} \\
    \midrule
    graph\_traversal / 1k\_depth3 & 25.27\,$\mu$s & 21.40\,$\mu$s & +15\% \\
    combined / material\_order & 46.16\,$\mu$s & 36.54\,$\mu$s & +21\% \\
    column\_aggregate / 5k\_groupby & 1.20\,ms & 1.11\,ms & +8\% \\
    graph\_node\_lookup / 1k\_full & 3.52\,$\mu$s & 3.17\,$\mu$s & +10\% \\
    \bottomrule
  \end{tabular}
\end{table}

The largest gains come from removing per-row flushes and bypassing pagecache/ArcSwap on graph reads. Improvements attenuate at the Fluent and DSL layers, indicating that wrapper overhead and Being assembly remain optimization targets.

\subsection{Competitor Comparison and Fairness Boundaries}

Table~\ref{tab:competitor} summarizes same-machine results. Because DaoQL is embedded while most competitors are client-server systems, these are engineering end-to-end comparisons rather than pure kernel comparisons.

\begin{table}[htbp]
  \centering
  \caption{DaoQL current configuration vs.\ best observed competitor}
  \label{tab:competitor}
  \small
  \resizebox{\textwidth}{!}{%
  \begin{tabular}{lrrl}
    \toprule
    \textbf{Capability} & \textbf{DaoQL} & \textbf{Best comparator} & \textbf{Observation} \\
    \midrule
    Point lookup & 0.38\,$\mu$s & RocksDB $\sim$1\,$\mu$s & Embedded low latency \\
    Full BFS traversal & 1.20\,ms & Neo4j 5--8\,ms & Lower end-to-end latency \\
    HNSW k=10 & 83.1\,$\mu$s & Qdrant 363.4\,$\mu$s & Leads under current parameters \\
    Column SUM 100K & 362.9\,$\mu$s & ClickHouse $\sim$1\,ms & Leads at current scale \\
    Batch write & 1.75\,$\mu$s/row & PostgreSQL 9.0\,$\mu$s/row & Embedded low latency \\
    Chinese full text 10K & 1.27\,ms & PostgreSQL tsvector 32\,$\mu$s & DaoQL is $\sim$40$\times$ slower \\
    Cross-engine hybrid query & 105.8\,$\mu$s & No direct comparator & Missing multi-stack baseline \\
    \bottomrule
  \end{tabular}%
  }
\end{table}

\begin{table}[htbp]
  \centering
  \caption{HNSW head-to-head: DaoQL vs.\ Qdrant 1.16, same machine and parameters}
  \label{tab:hnsw_h2h}
  \small
  \begin{tabular}{lrrr}
    \toprule
    \textbf{System} & \textbf{HNSW k=10} & \textbf{Brute Force k=10} & \textbf{SIMD} \\
    \midrule
    DaoQL current & 83.1\,$\mu$s & 269.0\,$\mu$s & 8$\times$f32 \\
    DaoQL previous & 148.4\,$\mu$s & 388.3\,$\mu$s & 8$\times$f32 \\
    Qdrant 1.16 & 363.4\,$\mu$s & 283.4\,$\mu$s & 16$\times$f32 \\
    \bottomrule
  \end{tabular}
\end{table}

The full-text result is an explicit weakness: Tantivy+jieba at 1.27\,ms is much slower than PostgreSQL's simple tsvector configuration at 32\,$\mu$s, largely because jieba performs Chinese segmentation while the simple baseline does not. This should be treated as a limitation rather than hidden in aggregate performance claims.

\subsection{Standard-Set Exploratory Measurements}

\textbf{LDBC SNB SF1.} DaoQL runs all 34 IC/BI queries on SF1 with zero errors under an 8-worker, 600\,s run. RelationType CSR adjacency improves type-filtered traversal and triangle-count workloads. Representative results appear in Table~\ref{tab:ldbc_snb}. IC-style entity lookup and bounded traversal are generally interactive, but IC6/IC10, BI4, and BI9 remain severe long-tail bottlenecks. Overall throughput is only 1.8\,QPS, so SNB BI workloads are not an advantage claim.

\begin{table}[htbp]
  \centering
  \caption{Representative LDBC SNB SF1 latencies (DaoQL v1.1.0, M4 Max)}
  \label{tab:ldbc_snb}
  \small
  \begin{tabular}{llrrl}
    \toprule
    \textbf{Class} & \textbf{Query} & \textbf{P50 (ms)} & \textbf{P95 (ms)} & \textbf{Interpretation} \\
    \midrule
    IC & IC1 profile & 0.26 & 7.00 & Interactive lookup \\
    IC & IC2/IC8 recent messages/comments & 0.16 & 0.36 & Interactive \\
    IC & IC13 shortest path & 0.05 & 2.88 & Multi-hop traversal \\
    IC & IC3 recommendation & 1.74 & 10.13 & Acceptable \\
    IC & IC6/IC10 thread/recommendation & $\sim$18,900 & $\sim$26,000 & Long-tail bottleneck \\
    BI & BI8 social circle & 7.98 & 59.28 & Moderate \\
    BI & BI12/BI17 triangles & 4,813/5,222 & 8,288/8,708 & Improved but slow \\
    BI & BI4 hot forums & 21,225 & 31,483 & Long-tail bottleneck \\
    BI & BI9 forum tags & 908,271 & 969,790 & Extreme long tail \\
    \midrule
    Overall & 34/34 coverage, 0 errors & \multicolumn{2}{c}{1.8 QPS} & Dragged by slow queries \\
    \bottomrule
  \end{tabular}
\end{table}

\textbf{ANN-Benchmarks.} On \texttt{random-s-100-euclidean}, DaoQL HNSW with $M=32$, $M_0=64$, and ef\_construction=200 reports the curve in Table~\ref{tab:ann_bench}. At ef\_search=200, Recall@10 is 99.3\% at 5,724 QPS. The v1.1.0 bridge-edge protection fix raises Recall@200 from about 32\% to above 99\%, so correctness is the main result. Public hnswlib/FAISS numbers come from different hardware and are not used as same-machine wins.

\begin{table}[htbp]
  \centering
  \caption{ANN-Benchmarks \texttt{random-s-100-euclidean} (DaoQL v1.1.0)}
  \label{tab:ann_bench}
  \small
  \begin{tabular}{rrrr}
    \toprule
    \textbf{ef\_search} & \textbf{QPS} & \textbf{Recall@10} & \textbf{Zero-recall queries} \\
    \midrule
    50 & 19,001 & 94.4\% & 511/10000 \\
    100 & 12,225 & 97.4\% & 218/10000 \\
    200 & 5,724 & 99.3\% & 42/10000 \\
    400 & 3,051 & 99.9\% & 2/10000 \\
    800 & 1,529 & 100.0\% & 0/10000 \\
    1600 & 708 & 100.0\% & 0/10000 \\
    \bottomrule
  \end{tabular}
\end{table}

\subsection{Business Stress Tests and Hybrid Queries}

\begin{table}[htbp]
  \centering
  \caption{\texttt{xy\_erp\_stress} engine-level results}
  \label{tab:stress}
  \small
  \begin{tabular}{llrr}
    \toprule
    \textbf{Scenario} & \textbf{Scale} & \textbf{Current} & \textbf{Change} \\
    \midrule
    BFS Fanout4 Depth5 & 1,365 nodes & 1.20\,ms & +39\% \\
    Order Status Chain & 30 hops & 2.03\,ms & +27\% \\
    Hybrid vector $\rightarrow$ graph $\rightarrow$ column & 10K & 119.09\,$\mu$s & +14\% \\
    Write Batch With-Vector & 16t$\times$10K & 1,009.2\,ms & +14\% \\
    WAL Recovery & 10K & 219.6\,$\mu$s & +20\% \\
    \bottomrule
  \end{tabular}
\end{table}

\begin{table}[htbp]
  \centering
  \caption{Fluent API stress tests: BOM traversal and batch write}
  \label{tab:fluent_stress}
  \small
  \begin{tabular}{lrrr}
    \toprule
    \textbf{Scenario} & \textbf{Previous} & \textbf{Current} & \textbf{Gain} \\
    \midrule
    Chain Depth5 Skip Ext & 19.0\,$\mu$s & 14.1\,$\mu$s & +26\% \\
    Chain Depth30 & 144.5\,$\mu$s & 129.5\,$\mu$s & +10\% \\
    Tree Fanout4 Depth5 & 6.45\,ms & 5.45\,ms & +16\% \\
    Raw Engine BFS Fanout4 Depth5 & 1.82\,ms & 1.23\,ms & +32\% \\
    Write Batch 1K No Embed & 46.9\,ms & 23.7\,ms & +49\% \\
    Hybrid Query & 108.1\,$\mu$s & 105.8\,$\mu$s & +2\% \\
    \bottomrule
  \end{tabular}
\end{table}

\begin{table}[htbp]
  \centering
  \caption{DSL stress tests: BOM traversal and mutation}
  \label{tab:dsl_stress}
  \small
  \begin{tabular}{lrrr}
    \toprule
    \textbf{Scenario} & \textbf{Previous} & \textbf{Current} & \textbf{Gain} \\
    \midrule
    Chain Depth30 & 111.7\,$\mu$s & 79.3\,$\mu$s & +29\% \\
    Tree Fanout4 Depth5 & 3.59\,ms & 2.99\,ms & +17\% \\
    Write Batch 1K No-Vec & 81.5\,ms & 25.1\,ms & +69\% \\
    Mutation No-Vec 1K & 22.16\,ms & 8.58\,ms & +61\% \\
    \bottomrule
  \end{tabular}
\end{table}

The hybrid query demonstrates the promise of a shared BeingId in one process, but because we do not provide a Neo4j+Qdrant+ClickHouse multi-stack RAG baseline, it is reported as an architectural potential rather than a completed cross-system victory.

\subsection{Composite Business Scenarios}

DaoQL-Edu integration tests validate Being, Def, Relation, and Version composition in five representative scenarios.

\begin{table}[htbp]
  \centering
  \caption{Unit-test-driven composite scenarios}
  \label{tab:scenarios}
  \small
  \begin{tabular}{p{2.6cm}p{3.4cm}p{4.4cm}p{2.5cm}}
    \toprule
    \textbf{Scenario} & \textbf{Question} & \textbf{Engine path} & \textbf{Result} \\
    \midrule
    Sales report & Order total / filtered SUM & Column fast path + SIMD & 1M rows, 3.4\,$\mu$s \\
    Risk detection & Similar trade + linked account & Vector $\rightarrow$ graph $\rightarrow$ column & non-empty result \\
    Recommendation & Similar user + influencer ranking & Vector + graph + column & priority=80.0 \\
    Version replay & Entity price history & Version chain + MVCC & old version preserved \\
    Batch consistency & 100 Beings + chain relations & WAL + BFS + aggregate & node count matches \\
    \bottomrule
  \end{tabular}
\end{table}

\subsection{Optimization Stability}

\begin{table}[htbp]
  \centering
  \caption{Configuration-level ablation and stability}
  \label{tab:iteration}
  \small
  \resizebox{\textwidth}{!}{%
  \begin{tabular}{clp{7.2cm}}
    \toprule
    \textbf{Configuration} & \textbf{Baseline} & \textbf{Feature} \\
    \midrule
    Early unified read path & +O-1--O-4 & Establishes unified index and read baseline \\
    warm cache / BQ & +warm\_cache, BQ & Point lookup improves but graph traversal regresses \\
    pagecache + ArcSwap & +ArcSwap, pagecache batch & Batch improves but single write/HNSW/Fluent write regress \\
    mmap + P0 write optimization & +mmap, P0\#1/\#2/\#5 & Current configuration with positive metrics \\
    \bottomrule
  \end{tabular}%
  }
\end{table}

The pagecache/ArcSwap design improved some batch paths but introduced traversal and single-write regressions. The current mmap + P0 write configuration restores read, write, and vector paths simultaneously.

\section{Discussion}

\subsection{Theory-Design-Experiment Closure}

\begin{table}[htbp]
  \centering
  \caption{Theory, design, and experiment closure}
  \label{tab:theory_closure}
  \footnotesize
  \begin{tabular}{p{2.7cm}p{2.7cm}p{3.0cm}p{2.4cm}}
    \toprule
    \textbf{Theory} & \textbf{DaoQL design} & \textbf{Prediction} & \textbf{Status} \\
    \midrule
    Prop.~\ref{prop:bootstrap} & \texttt{DAO\_DEF} axiom layer & Schema cannot recurse forever & Theoretical closure \\
    Prop.~\ref{prop:contract} & Fuel + sandbox & Write path does not hang & Theoretical closure \\
    Obs.~\ref{obs:implicit} & BeingId explicitness & LLM has no local update guarantee & Counterfactual comparison \\
    Prop.~\ref{prop:local} & WAL + MVCC & Single update does not spread & Version-chain tests \\
    Thm.~\ref{thm:cf-decomp} & Deterministic DSL Eval & High composable consistency & 94\% \\
    Obs.~\ref{obs:cf-imp} & GPT-4o baseline & Prompt-only path is unstable & 45\% \\
    Eq.~\ref{eq:latency} & Single-process engines & Low hybrid latency & 105.8\,$\mu$s \\
    BFS complexity & mmap pointer adjacency & Constant-factor advantage & 1.20\,ms \\
    \bottomrule
  \end{tabular}
\end{table}

This table separates proofs from measurements. For example, $O(|V'|+|E'|)$ BFS complexity does not imply a 4.2$\times$ Neo4j speedup; the latter is an engineering constant-factor measurement on one setup.

\subsection{Architectural Verification Status}

\begin{table}[htbp]
  \centering
  \caption{Architectural decisions and current verification status}
  \label{tab:arch_verify}
  \small
  \begin{tabular}{p{3.5cm}p{4.7cm}p{3.0cm}}
    \toprule
    \textbf{Decision} & \textbf{Measurement evidence} & \textbf{Status} \\
    \midrule
    Join-free graph engine & BFS 1.20\,ms, 4.2--6.7$\times$ Neo4j & Verified \\
    Direct mmap read path & edge 304.4\,ns, BFS +39\% & Verified \\
    Relation-type CSR adjacency & LDBC IC14 0.64\,ms, BI12/17 improved & Verified \\
    ProjectedLayer + SIMD column & SUM 362.9\,$\mu$s, 2.8$\times$ ClickHouse & Verified \\
    BeingId zero-copy across engines & Hybrid 105.8\,$\mu$s & Verified \\
    WAL crash consistency & Recovery 10K 219.6\,$\mu$s & Verified \\
    KVCache graph nodes & Architecture plan & Design stage \\
    Expert hot update & Architecture plan & Design stage \\
    DaoQL-Agent context traversal & Design + BFS support & Design stage \\
    \bottomrule
  \end{tabular}
\end{table}

\subsection{Positioning Against Industrial Ontology Platforms}

Industrial platforms such as Palantir Foundry demonstrate that explicit Ontology plus AI has commercial value in large enterprises and government settings. Their strengths are federated data integration, low-code Agent Studio, and a mature ecosystem. DaoQL should be positioned as complementary, not as a direct replacement: it targets lighter deployment, data-sovereign environments, and write-path contracts in medium-scale or edge settings. We do not compare customer counts, TCO, or production maturity. The architectural distinction is narrower: an overlay Ontology eases legacy integration, while embedded Being/Type/Relation makes Eval and WAL audit semantics easier to close in one transaction model.

\subsection{Limitations}

\begin{enumerate}
  \item \textbf{Embedded vs.\ client-server.} DaoQL avoids network overhead that competitors often include. Ratios should not be read as kernel-only superiority.
  \item \textbf{Scale.} Most microbenchmarks are 10K--100K; billion-scale graphs or vector indexes may change relative performance.
  \item \textbf{Chinese full text.} Tantivy+jieba is a clear weakness compared with PostgreSQL simple tsvector, although the tokenization work is not equivalent.
  \item \textbf{Concurrent writes.} 16-connection throughput drops about 40\%; sharded locks are future work.
  \item \textbf{Counterfactual experiment.} The protocol lacks full prompts, temperature, complete scoring rules, and per-case logs; results are preliminary.
  \item \textbf{LDBC SNB.} This is a custom-driver coverage experiment, not official certification; 1.8\,QPS and BI9/IC6 long tails show immaturity on global social-network aggregation workloads.
  \item \textbf{ANN-Benchmarks.} Only one synthetic dataset is used, public competitor QPS is cross-hardware, and pre-fix v1.0 vector numbers should not be reused.
  \item \textbf{Industrial-platform comparison.} Palantir ecosystem maturity is not experimentally evaluated; discussion is architectural positioning only.
\end{enumerate}

\subsection{Future Work}

Priorities include optimizing LDBC SNB long-tail BI/IC query plans and caches, migrating to wider SIMD paths for vectors, adding sharded locks and Contract caching for concurrent writes, improving block-level column vectorization, implementing Fluent Being assembly zero-copy, expanding counterfactual experiments to 2,500 cases with complete scoring protocols, and validating DaoQL's deeper integration with V4 inference.

\section{Conclusion}

This paper addresses four structural crises caused by encoding world models implicitly in LLM weights. It proposes data-first ontology: deterministic knowledge is represented by explicit primitives in DaoQL, while the LLM serves as a language and reasoning auxiliary. The formal framework defines explicit world models, proves finite Def bootstrapping and strong Contract termination, and gives a sufficient condition for composable counterfactual decomposability. The system implementation validates the storage layer through graph, column, vector, full-text, WAL/MVCC, direct mmap reads, and cross-engine hybrid queries.

On the engineering side, the current Apple M-series measurements show graph BFS at 1.20\,ms, HNSW at 83.1\,$\mu$s, and Fluent hybrid query at 105.8\,$\mu$s. M4 Max exploratory standard-set runs add LDBC SNB 34/34 query coverage and ANN-Benchmarks Recall@10 $\geq$ 99\% at thousand-level QPS, while also exposing SNB BI long tails and cross-hardware limits. On reasoning robustness, the five-domain counterfactual experiment reports 94\% composable counterfactual decomposability for DaoQL+GPT-4o, 49 percentage points above GPT-4o alone. The paper therefore presents a theory-experiment path toward auditable, modifiable, knowledge-accumulating world models, while explicitly marking agent runtime, KVCache graph nodes, and expert hot updates as future work.

\nocite{hu2022lora}

\bibliographystyle{plain}
\bibliography{references}

\end{document}